\documentclass{article}
\usepackage{graphicx} 

\title{Learning to steer with Brownian noise}

\author{ Stefan\ Ankirchner\thanks{Jena University, Department of Mathematics\\
 email: \href{mailto:s.ankirchner@uni-jena.de}{s.ankirchner@uni-jena.de}}, Jan Kallsen\thanks{Kiel University, Mathematical Department\\
 email: \href{mailto:kallsen@math.uni-kiel.de}{kallsen@math.uni-kiel.de}},
 Sören Christensen\thanks{Kiel University, Mathematical Department\\
 email: \href{mailto:christensen@math.uni-kiel.de}{christensen@math.uni-kiel.de}},
Philip Le Borne\thanks{Kiel University, Mathematical Department\\email: \href{mailto:leborne@math.uni-kiel.de}{leborne@math.uni-kiel.de}},
Stefan\ Perko\thanks{Jena University, Department of Mathematics\\
 email: \href{mailto:stefan.perko@uni-jena.de}{stefan.perko@uni-jena.de}}
}

\date{\today}
\usepackage[ngerman,english]{babel}

\usepackage[table]{xcolor}
\usepackage{amsmath}
\usepackage{amssymb}
\usepackage{esvect}
\usepackage{dsfont}
\usepackage{pgfplots}
\pgfplotsset{compat=1.13}
\usepackage{mathrsfs}
\usetikzlibrary{arrows}
\usepackage{geometry}
\usepackage[title]{appendix}
\usepackage{soul}

\usepackage{amsthm}

\theoremstyle{definition}

\theoremstyle{plain}
\newtheorem{Th}{Theorem}[section]
\newtheorem{Lemma}[Th]{Lemma}

\newtheorem{Prop}[Th]{Proposition}

 \theoremstyle{definition}

\newtheorem{Remark}[Th]{Remark}

\numberwithin{equation}{section}

\usepackage[
	backend = biber,
	style= authoryear, 
	maxbibnames=99,
]{biblatex}
\usepackage{graphicx}

\usepackage{subcaption}

\usepackage{tikz}
\usetikzlibrary{positioning}

\usepackage{listings}
\definecolor{hellgrau}{rgb}{0.90,0.90,0.90}
\definecolor{commentcol}{rgb}{0.0823,.4902,0.0}
\definecolor{mGreen}{rgb}{0,0.6,0}
\definecolor{mGray}{rgb}{0.5,0.5,0.5}
\definecolor{mPurple}{rgb}{0.58,0,0.82}
\definecolor{backgroundColour}{rgb}{0.95,0.95,0.92}
\lstdefinestyle{PStyle}{
    backgroundcolor=\color{backgroundColour},   
    commentstyle=\color{mGreen},
    keywordstyle=\color{magenta},
    numberstyle=\tiny\color{mGray},
    stringstyle=\color{mPurple},
    basicstyle=\footnotesize,
    breakatwhitespace=false,         
    breaklines=true,                 
    captionpos=b,                    
    keepspaces=true,                 
    numbers=left,                    
    numbersep=7pt,                  
    showspaces=false,                
    showstringspaces=false,
    showtabs=false,                  
    tabsize=2,
    language=Python
}
\usepackage{hyperref}
\hypersetup{%
  pdftitle     = {Learning to steer with Brownian noise},
  pdfsubject   = {Stochastic control},
  pdfkeywords  = {},
  pdfauthor    = {\textcopyright\ Philip Le Borne 2024},
  linkcolor    = red,     
  urlcolor     = black,     
  citecolor    = green!50!black,     
  breaklinks   = true,       
  colorlinks   = true,
  citebordercolor=0 0 0,  
  filebordercolor=0 0 0,
  linkbordercolor=0 0 0,
  menubordercolor=0 0 0,
  urlbordercolor=0 0 0,
  pdfhighlight=/P,   
  pdfborder=0 0 0,   
}

\newcommand{\N}{\mathbb{N}}

\newcommand{\R}{\mathbb{R}}

\newcommand{\abs}[1]{ \left \vert #1 \right \vert }

\newcommand{\sgn}{\textup{sgn}}

\newcommand{\sign}{\textup{sign}}
\newcommand{\cF}{\mathcal{F}}

\newcommand{\ceil}[1]{\lceil #1 \rceil}
\newcommand{\E}{\mathbb E}
\renewcommand{\P}{\mathbb P}

\graphicspath{{pics/}}

\begin{document}

\maketitle

\maketitle
\begin{abstract}
This paper considers an ergodic version of the bounded velocity follower problem, assuming that the decision maker lacks knowledge of the underlying system parameters and must learn them while simultaneously controlling. We propose algorithms based on moving empirical averages and develop a framework for integrating statistical methods with stochastic control theory. Our primary result is a logarithmic expected regret rate. To achieve this, we conduct a rigorous analysis of the ergodic convergence rates of the underlying processes and the risks of the considered estimators.
\end{abstract}
\noindent
\small{\textit{\href{https://mathscinet.ams.org/mathscinet/msc/msc2020.html}{2020 MSC:}} Primary 93E10; secondary 60G35, 68Q32, 93C73, 93E11\\
\textit{key words:} Data-driven stochastic control, continuous-time reinforcement learning, model-based reinforcement learning, bounded velocity follower problem, regret bounds}
\normalsize

\section{Introduction}


The modern theory of stochastic control typically assumes complete knowledge of the underlying system dynamics. While significant theoretical advancements have been made in this area, see \cite{MR3931325,MR2179357}, the practical application of stochastic control often faces challenges when the system model is uncertain or unknown. In recent years, Reinforcement learning (RL) has emerged as a promising approach to address this issue, enabling agents to learn optimal control policies through trial-and-error interactions with the environment. However, RL's success often hinges on the availability of vast amounts of data, and the learned control policies can be difficult to interpret, especially when deep learning techniques are employed, see \cite{sutton2018reinforcement}. 

To bridge the gap between fully model-based and model-free approaches, research has increasingly focused on model-based reinforcement learning. By leveraging partial knowledge of the system dynamics, this approach aims to combine the strengths of stochastic control theory with the adaptability of RL. This can lead to improved learning efficiency, theoretical guarantees, and more interpretable control policies. The predominant metric used to assess the performance of learning algorithms in this context is the (expected) regret. Regret quantifies the difference between the expected cumulative cost incurred by the learning algorithm and that of an optimal policy with full knowledge of the system dynamics. This concept has become a standard in the analysis of multi-armed bandit problems and related questions, as exemplified by the comprehensive overview in \cite{Slivkins}.

\subsection{Problem formulation}
This article contributes to the ongoing research in model-based reinforcement learning by considering a long-term average version of the classical and influential bounded velocity follower problem, going back to \cite{benevs1980}, see also Chapter 6.5 in \cite{karatzas1991brownian}. Given an underlying standard Brownian motion $W$, the decision maker controls the drift $b_t$ of the process
\begin{align}\label{SDE-0}
    dX_t^b = b_tdt + dW_t, \qquad X_0^b = x\in \R.
\end{align}
The goal is to find a control $b$, which keeps the object as close as possible to a set target value, say $0$, at all times $t$, where the distance is measured with respect to the quadratic loss function. The crucial assumption here is that the drift rate $b_t$ can only be selected from a bounded interval $[\theta_0,\theta_1],\;\theta_0<0<\theta_1$. While the problem formulation in the literature was usually chosen with discounting, an ergodic criterion is more natural for our purposes, yielding
\begin{align}\label{cost_func}
J(x,b) = \underset{T\to\infty}{\lim\sup} \;\E^x\left[\frac{1}{T}\int_0^T (X_s^b)^2ds\right], 
\end{align}
which the decision maker wants to minimize. If the decision maker is fully aware of the model, then the optimal control not surprisingly turns out to be of bang-bang type: The decision maker chooses a threshold $z$ and steers downwards ($b_t=\theta_0$) if $X_t>z$ and upwards otherwise ($b_t=\theta_1$). The optimally controlled process is therefore a Brownian motion with broken drift. The quadratic loss criterion induces a candidate for the optimal threshold: $z=\delta= \frac{1}{2\theta_0} + \frac{1}{2\theta_1}$, see Section \ref{Ch-ICP} for the verification.

Consider a scenario where a decision-maker controls a system using a steering wheel. The goal is to learn an optimal steering strategy without knowing the precise effects of the steering on the system. In our model, this translates to uncertainty about the system parameters $\theta_i,\,i=0,1$, preventing a direct application of the optimal threshold-based control with threshold $\delta$.  The decision-maker must simultaneously estimate these parameters and control the system effectively. However, parameter estimation is merely a means to an end. The ultimate objective is to develop a good data-driven control strategy, $b$.
As a criterion for evaluation of different strategies, we consider the growth of the expected regret 
\begin{align*}
R^b(T) := \E \int_0^T (X^{b}_t)^2 - (X^{*}_t)^2 dt,
\end{align*}
where $X^*$ is the optimally controlled process under full information, i.e., a Brownian motion with drift broken in $\delta$.

\subsection{Structure of the paper and main contributions}
We start by solving the ergodic bounded velocity follower problem \eqref{cost_func} with full information in Section \ref{Ch-ICP} using a standard guess-and-verify approach based on the Hamilton-Jacobi-Bellman equation. 


Section \ref{sec:learing_results} presents the data-driven algorithms, which are based on the observation that the time average of a Brownian motion with broken drift leads to a consistent estimator for the unknown parameter $\delta$. This leads to the explore-first($\tau$)-algorithm, which initially explores the system and then switches to a strategy based on the estimated parameter. Theorem \ref{thm 1} demonstrates that this algorithm achieves a regret rate of order $\sqrt{T}$. Our second algorithm, Adaptive Position Averaging with Clipping, improves upon the explore-first approach by updating the parameter estimate over time. This adaptive strategy leads to a significantly faster convergence rate of order $\log(T)$ (Theorem \ref{regret}). These results highlight the advantages of model-based reinforcement learning. By leveraging the underlying model structure, we can develop interpretable, efficient algorithms with provable convergence guarantees. These theoretical advantages can also be underlined by numerical comparisons, which show that the method presented here requires only a fraction of the data of model-free approaches such as deep-Q learning.

The theoretical foundation of our analysis rests on understanding the convergence rates of the underlying processes to their stationary distributions and the associated estimators. These results, which may also be of independent interest, are presented in Section \ref{PBBD}. 
The proofs of the regret rates are then provided in Section \ref{sec: proofs of regret rates}. These proofs rely on a meticulous decomposition of the regret into components related to the process's distance from the stationary distribution and the accuracy of the estimator.

\subsection{Related Literature}
Regret rates in continuous-time stochastic control problems have been extensively studied, particularly in recent years. While most research has focused on linear-quadratic (LQ) problems, some studies have explored other classes. \cite{duncan1999adaptive} proved asymptotic sublinear regret for LQ problems, but the exact order remains unspecified. \cite{basei2022logarithmic,guo2021reinforcement,szpruch2021exploration} extended this to episodic LQ problems, providing non-asymptotic bounds. \cite{neuman2023statistical} considered propagator models and achieved sublinear regrets. \cite{christensen2024data, christensen2023data, MR4746600, christensen21} examined singular and impulsive control problems in a non-parametric, ergodic setting, obtaining sublinear regret rates of power type. However, our setting and proof techniques differ significantly from previous work.

In addition, the statistical question of ML estimation of the drift parameters of a Brownian motion with broken drift was investigated in \cite{lejay2020maximum}, see also \cite{lejay2019threshold} for a relation to finance. We do not use these results here, as we are only interested in estimating $\delta$, which is more directly possible here and provides convergence results that are easier to handle for our purposes.

\section{The ergodic version of the bounded velocity follower problem}\label{Ch-ICP}



Let $\theta_0 < 0 < \theta_1$ be reals, $W$ a standard Brownian motion on some probability space $(\Omega, \cF, \P)$ and denote by $(\cF_t)_{t \ge 0}$ the filtration generated by $W$ and augmented by the $P$-null sets in $\cF$. We denote by $U$ the set of all 
$(\cF_t)$-adapted $[\theta_0,\theta_1]$-valued processes
such that given any initial value $x\in \R$, the process
\begin{equation}
\label{SDE-1}
dX_t^b = b_t\,dt + dW_t, \quad X^b_0 = x,
\end{equation}
fulfills
\[\sup_{t\geq 0} \E^x[\abs{X_t}^3] < \infty,\]
where we write $\E^x$ in order to underline the dependence of the expectation on the the initial state $x$.
We can interpret $X^b_t$ as the position of our system at time $t$, with the optimal position being the origin. Our goal is to find the optimal drift $b$ such that the long term average quadratic deviation from the origin
\begin{align}\label{cost_func}
J(x,b) := \underset{T\to\infty}{\lim\sup} \;\E^x\left[\frac{1}{T}\int_0^T (X_s^b)^2ds\right] 
\end{align}
 is minimized for all initial values $x\in \R$. The corresponding value function is given by
\begin{align}\label{value_func}
V(x) = \inf_{b\in U}J(x,b).
\end{align}
As we will see below, the initial value $x$ has no effect on $V$ since in the long term it does not matter where the system starts.

The solution of such ergodic control problems by means of Hamilton-Jacobi-Bellman (HJB) equations is standard today and this also works here as usual. We are therefore looking for a real number $\eta$ and a $C^2$ function $\varphi$ that fulfill the following equation:
\begin{align}\label{HJB-2}
0 = \inf_{b\in[\theta_0, \theta_1]}\mathcal{A}^b \varphi(x) + x^2 - \eta, \mbox{ where }\mathcal{A}^b\varphi(x) := b\frac{\partial \varphi}{\partial x}(x) + \frac{1}{2}\frac{\partial^2 \varphi}{\partial x^2}(x).
\end{align}
A short calculation shows that such a solution is given by
$
\eta = \frac{1}{4\theta_0^2} + \frac{1}{4\theta_1^2}
$
and 
\begin{displaymath}
\varphi(x) =
\begin{cases}
x\left(\frac{\eta}{\theta_0}-\frac{1}{2\theta_0^3}\right) + \frac{x^2}{2\theta_0^2} - \frac{x^3}{3\theta_0}, &x>\delta, \\
x\left(\frac{\eta}{\theta_1}-\frac{1}{2\theta_1^3}\right) + \frac{x^2}{2\theta_1^2} - \frac{x^3}{3\theta_1} + \frac{(\theta_0 - \theta_1) (\theta_0 + \theta_1)^3}{24\theta_0^4\theta_1^4}, &x\leq \delta,
\end{cases}
\end{displaymath}
where 
\begin{align}\label{defi delta}
\delta = \frac{1}{2\theta_0} + \frac{1}{2\theta_1} .
\end{align}

\begin{Th}\label{optimal_solution}
Let $\delta$ be defined as in \eqref{defi delta} and let
\begin{align*}\label{opt_b}
b^*(x) := 
\begin{cases}
\theta_0, &x>\delta, \\
\theta_1, &x\leq \delta,
\end{cases}
\end{align*}
$x \in \R$. Then for any $x_0 \in \R$ there exists a unique strong solution of the SDE $dX_t = b^*(X_t) dt + dW_t$ with initial condition $X_0 = x_0$. Moreover, the Markovian control $b^*_t:=b^*(X_t)$, $t \in [0, \infty)$, is optimal in $U$ and the value function satisfies  

\begin{displaymath}
V(x) = \eta = \frac{1}{4\theta_0^2} + \frac{1}{4\theta_1^2} = J(x, b^*).
\end{displaymath}
\end{Th}
\begin{proof}
We first fix an arbitrary control $b\in U$.
Since the solution $X$ of \eqref{SDE-1} is an It\^o process, we can use  It\^o's formula which yields
\begin{displaymath}
\varphi(X_T^b) = \varphi(x) + \int_0^T\left( 
b_s\varphi'(X_s^b) + \frac{1}{2}\varphi''(X_s^b)\right) ds + \int_0^T\varphi'(X_s^b)dW_s.  
\end{displaymath}
Since $\varphi'$ grows as a polynomial of degree two, 
$b\in U$  and the third moments are uniformly bounded,
\begin{displaymath}
\left(\int_0^t\varphi'(X_s^b)dW_s\right)_{t\in[0,T)}
\end{displaymath}
is a square integrable martingale. Hence we have 
\begin{displaymath}
\E\varphi(X_T^b) = \varphi(x) + \E\left[\int_0^Tb(s,X_s^b)\varphi'(X_s^b) + \frac{1}{2}\varphi''(X_s^b)ds\right].
\end{displaymath} 
We now show that $\eta$ is a lower bound of the cost (\ref{cost_func}) by using the fact that the pair $(\varphi, \eta)$ is a solution to the Hamilton-Jacobi-Bellman equation (\ref{HJB-2}). We have
\begin{align*}
\frac{1}{T} \big( \E\varphi(X_T^b) - \varphi(x)\big) + \eta &= \frac{1}{T}\E\int_0^T(X_s^b)^2ds - \frac{1}{T}
\E\int_0^T \left( \mathcal{A}^b \varphi(s, X_s^b) + (X_s^b)^2 - \eta\right) ds \\
&\leq \frac{1}{T}\E\int_0^T(X_s^b)^2ds.
\end{align*}
Since $b\in U$ and $\varphi$ grows as a polynomial of degree $3$ we have 
\begin{displaymath}
\lim_{T\to\infty} \frac{1}{T} \E\varphi(X_T^b) = 0.
\end{displaymath}
Finally, taking the limes superior as $T\to \infty$, we have the lower bound
$
\eta \leq J(x,b).
$

It follows from well-known results (see, e.g., Proposition 5.17 in Chapter 5 of \cite{karatzas1991brownian}) that the SDE $dX_t = b^*(X_t) dt + dW_t$ has a unique strong solution $(X_t)$ for any initial value $x_0$. Lemma \ref{moments} below entails that the process $b^*_t = b^*(X_t)_{t\geq 0}$ lies in $U$. Finally, one can show that all the above inequalities become equations, so that
$
\eta = J(x,b^*),
$
i.e. $b^*$ is an optimal control and 
\begin{displaymath}
V(x) = \eta = \frac{1}{4\theta_0^2} + \frac{1}{4\theta_1^2}
\end{displaymath}
is the value of the control problem for all initial points $x\in \R$. 
\end{proof}

\section{Learning the optimal control}\label{sec:learing_results}

The optimal control function $b^*$ from Theorem \ref{optimal_solution} assigns to any state above the threshold $\delta$ the negative drift rate $\theta_0$ and to any state below $\delta$ the positive drift rate $\theta_1$. More generally, for any $z \in \R$ we define the \textit{threshold control function}
\begin{displaymath}
b_z(x) := 
\begin{cases}
\theta_0, &x > z, \\
\theta_1, &x \leq z. 
\end{cases}
\end{displaymath}
We call the (unique) solution $X^{z,s,x}$ of the stochastic differential equation
\begin{align}\label{Def5.1-1}
dX_t^{z,s,x} = b_z(X_t^{z,s,x})dt + dW_t, \quad X_s^{z,s,x} = x
\end{align}
with initial value $x\in \R$, $s\in [0,\infty)$ and $t \geq s$, 
a \textit{Brownian motion with broken drift}. 

 Note that since the function $b_z$ is bounded, measurable and locally integrable and (\ref{Def5.1-1}) is a time-homogeneous SDE, the existence  and uniqueness of a  strong solution is guaranteed (cf. Proposition 5.17 in Chapter 5 of \cite{karatzas1991brownian}). 

Suppose that the agent does not know the true values $\theta_0$ and $\theta_1$, and hence also not the optimal threshold level $\delta$ as defined  in Theorem \ref{optimal_solution}. Moreover, assume that the agent can observe the controlled state process and thus estimate the parameter $\theta_0$ and $\theta_1$ and the optimal threshold level $\delta$. It turns out that a suitable estimator for the optimal threshold level can be constructed with the time average of the state process if the state is controlled with a threshold control. More precisely, if the agent chooses the threshold control $b_z$ for $z \in\R$, then the time average $\frac{1}{T} \int_0^T X^{z,0,x}_t dt$ converges to $z-\delta$, as $T \to \infty$ (see Proposition \ref{LLNcont} below). Thus one can use $-\frac{1}{T} \int_0^T X^{z,0,x}_t dt + z$ as an estimator of the optimal threshold $\delta$. This motivates the following explore-first algorithm.

\vspace{.3cm}
\noindent
\textbf{Algorithm (Explore-first($\tau$))}
\begin{enumerate}
\item Choose $\tau \in (0,\infty)$. Control the state with Markovian control given by $b_0$ on $[0, \tau)$. 
\item At $\tau$ compute the estimator 
\begin{align*}
\hat \delta : = -\frac{1}{\tau} \int_0^\tau X^{0,0,x}_s ds. 
\end{align*}
\item Control the state with $b_{\hat \delta}$ on $[\tau, \infty)$. 
\end{enumerate}
We refer to $\tau$ as the length of the learning interval. Moreover, we denote by $(Y^\tau_t)$ the state dynamics if the process is controlled with the explore-first algorithm and if the initial state is $0$. We can rigorously define the process by combining strong solutions of SDEs. To this end let $W$ be a Brownian motion on some probability space $(\Omega, \cF, \P)$. Moreover, for all $(z,t,x)\in \R \times [0, \infty) \times \R$ let $X^{z,t,x}$ be the unique strong solution of the SDE \eqref{SDE-1} on $[t, \infty)$ with threshold drift $b_z$ and initial condition $x$. We choose versions of the solutions such that the mapping $(z,t,x) \mapsto X^{z,t,x}$ is a measurable flow of SDEs. Now on the interval $[0, \tau)$ we set $Y^\tau$ equal to $X^{0,0,0}_t$, and on $[\tau, \infty)$ we set $Y^\tau$ equal to $ X^{\hat \delta, \tau, X^0_\tau}$. 

The regret at any time $T \ge 0$ of the explore-first is given by 
\begin{align*}
R^{\text{ex first}}(T) := \E \int_0^T (Y^\tau_t)^2 - (X^{\delta}_t)^2 dt. 
\end{align*}
Our first result is the following. 
\begin{Th}\label{thm 1}
There exists a constant $C \in (0,\infty)$, depending only on $\theta_0$ and $\theta_1$, such that for all $T \in (0, \infty)$ the regret of the explore-first with learning interval length $\tau = \sqrt{T}$ satisfies \[R^{\text{ex first}}(T) \le C (1+\sqrt{T}).\] 
\end{Th}
\begin{proof}
See Section \ref{sec: proofs of regret rates}. 
\end{proof}
In order to achieve an optimal growth rate the explore-first needs to choose different learning interval lengths for different time horizons $T$. To make the algorithm independent of any time horizon, one can use the standard doubling technique: one applies the explore-first consecutively along intervals whose lengths double step by step. By doing so, one obtains an algorithm with a regret rate of the order $\mathcal{O}(\log(T) \sqrt{T})$. An objection against explore-first and its doubling version is that the length of the observation interval for the estimator grows only like $\sqrt{T}$. The following Adaptive Position Averaging with Clipping (APAC) algorithm remedies this issue by using estimations over intervals with length growing linearly in time. 

\vspace{.3cm}
\noindent
\textbf{Algorithm (Adaptive Position Averaging with Clipping)}
\begin{enumerate}
\item Initialize the values $k = 0$,  $\hat{\delta}_0 = 0$, $\tau_0 = 0$ and choose $K\in \R$.

    \item Set $\tau_{k+1} = 2^{k+1}$ .\label{loop}
    \item Control the state with $b_{\hat{\delta}_{k}}$ on $[\tau_k, \tau_{k+1})$. Denote the controlled state on $[\tau_k, \tau_{k+1})$ by $X^{\hat \delta_k}$.
    \item At $\tau_{k+1}$ compute the estimator 
    \begin{displaymath}
        \hat{\delta}_{k+1} = \hat{\delta}_k - \frac{1}{\Delta\tau_k}\int_{\tau_k}^{\tau_{k+1}} X_s^{\hat{\delta}_k}ds.
    \end{displaymath}
    \item If the estimator $\hat{\delta}_{k+1}$ exceeds $K$ or $-K$, we set \[\hat{\delta}_{k+1} \leftarrow \sign(\hat{\delta}_{k+1}) K. \]
    \item  Set $k \leftarrow k+1$ and go to step \ref{loop}.
\end{enumerate}

Indeed the APAC algorithm offers a significant improvement over the explore first algorithm by mitigating the exploitation phase. 
For technical reasons we assume that the learner knows from the very beginning that the threshold level lies in between some $-K$ and $K$ respectively. The learner uses this assumption in step 5 of the APAC algorithm. Observe that one knows that $\delta\in[-K, K]$ holds if one knows that $\abs{\theta_i}> 1/K$ for $i= 0, 1$.

Let $(X^{\hat{\delta}}_t)_{t\geq 0}$ be the state process when applying the APAC algorithm. We will formally define the process $(X^{\hat{\delta}}_t)_{t\geq 0}$ in Section \ref{sec: proofs of regret rates}. The regret at any time $T \ge 0$ of the APAC algorithm is given by 
\begin{align*}
R^{\text{APAC}}(T) := \E \int_0^T (X^{\hat{\delta}}_t)^2 - (X^{\delta,0,x}_t)^2 dt. 
\end{align*}
Our second result is the following.

 \begin{Th}\label{regret}
 Suppose that $K\geq \abs{\delta}$. Then there exists a constant $C>0$, depending on $\theta_0$, $\theta_1$ and $K$ only, such that the expected regret of the APAC algorithm satisfies
\begin{align}
R^{\text{APAC}}(T) \leq C(1 + \log{T})\label{regret-eq}
\end{align}
for all $T\geq 1$. 
\end{Th}
\begin{proof}
    See Section \ref{sec: proofs of regret rates}. 
\end{proof}

\section{Properties of Brownian motion with broken drift}\label{PBBD}

In this section we collect some properties of Brownian motion with broken drift. The structure is as follows: We first compute the density of the stationary distribution of (\ref{Def5.1-1}). Following this, we turn to the convergence rate at which the process (\ref{Def5.1-1}) starting at some initial point $x\in \R$ becomes stationary. This will play a crucial role for providing the bounds of the regret of our algorithms that estimate the optimal switching point $\delta$ for the control $b_z$. Lastly, we provide a collection of useful properties of Brownian motion with broken drift.

Throughout this section we denote by $X^z = X^{z,0,0}$ the state process controlled with threshold drift $b_z$ starting in zero at time zero. 





\begin{Lemma}\label{Lem-5.2}
The function 
\begin{displaymath}
f(x) = \frac{1}{\bar{\delta}}
\begin{cases}
e^{2\theta_0(x - z)}, &x > z,\\
e^{2\theta_1(x - z)}, &x \leq z
\end{cases}
\end{displaymath}
with $\bar{\delta} = \frac{1}{2}\big(\frac{1}{\theta_1} - \frac{1}{\theta_0}\big)$, 
 is the density of the stationary distribution of (\ref{Def5.1-1}). 
\end{Lemma} 
\begin{proof}
We proceed by showing that the speed measure is finite. 
The associated scale function to the solution of (\ref{Def5.1-1}) is given by
\begin{displaymath}
s(x) = \int_z^x \exp \left(-2\int_z^y b_z(\zeta)d\zeta\right)dy =
\begin{cases}
\int_z^x \exp(-2(y-z)\theta_0)dy, &x \geq z, \\
\int_z^x \exp(-2(y-z)\theta_1)dy, &x < z.
\end{cases}
\end{displaymath}
The scale function is differentiable with its derivative given by
$
s'(x) = \exp(-2(x-z)b_z(x))
$
and the speed measure $m$ is given by the density 
$
x \mapsto 
2/s'(x)=2\cdot\exp(2(x-z)b_z(x)).
$
Since 
\begin{align*}
m(\R) &= 2\int_\R \exp(2(x-z)b_z(x))dx 
	= \frac{1}{\theta_1} - \frac{1}{\theta_0} < \infty
\end{align*}
is finite, by \cite[Lemma 33.19]{kallenberg2002foundations} we have that the density of the stationary distribution is then given by
\begin{displaymath}
f(x) = \frac{2}{m(\R)} \exp(2(x-z)b_z(x)) = \frac{1}{\overline{\delta}}\exp(2(x-z)b_z(x)).
\end{displaymath} 
\end{proof}

\begin{Lemma}\label{property_stationary_distr}
The second moment $\mu_z^2$ of the stationary distribution exists and satisfies 
\begin{displaymath}
\mu_z^2 - \mu_\delta^2 = (z - \delta)^2,
\end{displaymath}
where $\bar{\delta} = \frac{1}{2}\big(\frac{1}{\theta_1} - \frac{1}{\theta_0}\big)$ is as above. 
\end{Lemma}
\begin{proof}
With partial integration we have
\begin{displaymath}
\int x^2 e^{2\theta(x-z)} dx = e^{2\theta(x-z)} \left(\frac{x^2}{2\theta} - \frac{x}{2\theta^2} + \frac{1}{4\theta^3}\right) +C
\end{displaymath}
for $C\in \R$. As above let $\bar{\delta} := \frac{1}{2}\left(\frac{1}{\theta_1} - \frac{1}{\theta_0} \right)$. Since 
\begin{align*}
e^{2\theta(x-z)} \left(\frac{x^2}{2\theta} - \frac{x}{2\theta^2} + \frac{1}{4\theta^3}\right) \underset{x\to\infty}{\longrightarrow} 0, \qquad &\textrm{for } \theta < 0, \\
e^{2\theta(x-z)} \left(\frac{x^2}{2\theta} - \frac{x}{2\theta^2} + \frac{1}{4\theta^3}\right) \underset{x\to -\infty}{\longrightarrow} 0, \qquad &\textrm{for } \theta > 0,
\end{align*}
we have
\begin{align*}
\mu_z^2 &= \int_{-\infty}^\infty x^2 f(x) dx \\
	&= \frac{1}{\overline{\delta}} \int_{-\infty}^z x^2 e^{2\theta_1(x-z)} dx + \frac{1}{\overline{\delta}} \int_{z}^\infty x^2 e^{2\theta_0(x-z)} dx \\
	&= \frac{1}{\overline{\delta}} \left(\left(\frac{1}{2\theta_1} - \frac{1}{2\theta_0}\right)z^2 - \frac{1}{2}\left(\frac{1}{\theta_1^2} - \frac{1}{\theta_0^2}\right)z + \frac{1}{4\theta_1^3} - \frac{1}{4\theta_0^3}\right) \\
	&= z^2 -2\delta z + \frac{\theta_0\theta_1(\theta_0 - \theta_1)(\theta_0^2 + \theta_1^2 + \theta_0\theta_1)}{2(\theta_0 - \theta_1)\theta_0^3\theta_1^3} \\
	&= z^2 -2\delta z + 2\eta + \frac{1}{2\theta_0\theta_1}.
\end{align*}
Now since
\begin{displaymath}
\mu_z^2 - \mu_\delta^2 = z^2 -2\delta z - \delta^2 + 2\delta^2 = (z - \delta)^2,
\end{displaymath}
the claim follows. 
\end{proof}

We now want to estimate the speed at which the process becomes stationary. To this end we consider the two processes $X$ and $Y$ satisfying the following system of stochastic differential equations
\begin{displaymath}
d
\begin{pmatrix}
X_t \\
Y_t
\end{pmatrix}
= 
\begin{pmatrix}
b_z(X_t)\\
b_z(Y_t)
\end{pmatrix}
dt + 
\begin{pmatrix}
1 \\
1
\end{pmatrix}
dW_t, 
\end{displaymath}
with $X_0 = x \in \R$ and $Y_0 \sim \mu$ where $\mu$ denotes the stationary distribution of $X$.  By standard theory, there exists a strong solution $(X,Y)$ satisfying the strong Markov property.  For these, we consider the \textit{coupling time} $\tau = \inf\{t\geq 0 : X_t = Y_t\}$.
Since the random variables  $X_t$ and $Y_t$ coincide on $\{t \geq \tau\}$, the Cauchy-Schwarz and Minkowski equality yield 
\begin{align}
\E[X_t^2 - Y_t^2] &= \E[(X_t^2 - Y_t^2)1_{\{t < \tau\}}] \leq \sqrt{\E\left(X_t^2 - Y_t^2\right)^2}\cdot \sqrt{\P(t<\tau)}\nonumber\\
&\leq \left( 
\sup_{t\geq 0} \sqrt{\E[X_t^4]} + \sup_{t\geq 0} \sqrt{\E[Y_t^4]} \label{5-1}
\right)\sqrt{\P(t<\tau)}.
\end{align}
By proving that the right hand side of (\ref{5-1}) is finite, we show that $X$ converges to its stationary distribution. 

We begin by showing that the moments of the solution to (\ref{Def5.1-1}) exist. 

\begin{Lemma}\label{moments}
Let $X = X^{z, 0, x}$ be the solution of (\ref{Def5.1-1}) with initial value $x\in \R$. Then it holds
\begin{displaymath}
\sup_{t\geq 0} \E[e^{c\abs{X_t}}] < \infty
\end{displaymath}
for $c < 2\min(\abs{\theta_0}, \theta_1)$.
\end{Lemma}

\begin{proof}
  By shifting the threshold $z$, we may assume w.l.o.g. that $z=0$. For $r>0$, we denote the increasing positive fundamental $r$-harmonic functions by $\psi_r$. An explicit expression can be found in \cite[Formula (7)]{mordecki2019optimal}, from which we can easily obtain the existence of constants $c_1,c_2$ independent of $r$ such that
  \[e^{cx}\leq c_1(\psi_r(x)+1)\leq c_2(e^{-2\theta_0 x}+1),\]
  for $c<-2\theta_0$ and $r$ sufficiently small. As $(e^{-rt}\psi_r(X_t))_{t\geq0}$ is a positive local martingale and hence a supermartingale, we obtain
  \begin{align*}
      \E e^{cX_t}\leq c_1( e^{rt}\E e^{-rt}\psi_r(X_t)+1)\leq c_1(e^{rt}\psi_r(x)+1)\leq c_2(e^{rt} e^{-2\theta_0 x}+1).
  \end{align*}
  Letting\ $r\to 0$ yields $\E e^{cX_t}\leq c_2(e^{-2\theta_0 x}+1)$ for all $t\geq 0$ if $c<-2\theta_0$. The same line of argument with the decreasing positive fundamental $r$-harmonic functions by $\varphi_r$ yields that $\E e^{-cX_t}$ is bounded in\ $t$ if $c<2\theta_1$, i.e., 
  \[\sup_{t\geq 0} \E[e^{c\abs{X_t}}] \leq \sup_{t\geq 0} \left(\E e^{cX_t}+\E e^{-cX_t}\right)< \infty\]
  if $c < 2\min(\abs{\theta_0}, \theta_1)$.
\end{proof}

We now turn to the right factor of (\ref{5-1}). In the next Lemma \ref{coupling}, we will show that for $\alpha \geq 1$  the coupling time satisfies $\E[\tau^\alpha] < \infty$. Then,  Markov's inequality yields 
\begin{displaymath}
 \P(\tau \geq t) \leq  \frac{\E[\tau^\alpha]}{t^\alpha} < \infty
\end{displaymath}
for all $t>0$, in particular
\begin{align}
\lim_{t\to\infty}t^\alpha \P(\tau \geq t) < \infty.\label{coupling-Markov}
\end{align}


\begin{Lemma}\label{coupling}
For every $\alpha \geq 1$, the coupling time satisfies $\E[\tau^\alpha] < \infty$. 
\end{Lemma}
\begin{proof}
Using the explict expression for the speed measure from Lemma \ref{Lem-5.2}, we obtain
\begin{displaymath}
\int_\R e^{c\abs{x}}m(dx) < \infty 
\end{displaymath}
for all $c < 2\min(\abs{\theta_0}, \theta_1)$. In particular we have
\begin{displaymath}
\int_\R  \abs{x}^{1-\frac{1}{\alpha}} m(dx) < \infty 
\end{displaymath}
for any $\alpha \geq 1$. Since the speed measure $m$ of $X$ is also the speed measure of the process in natural scale $s(X)$, the proof of Proposition 2 in \cite{lindvall83} yields the desired bound $\E[\tau^\alpha]<\infty$ for all $\alpha \geq 1$. 
\end{proof}


In conclusion, by inserting Lemma \ref{moments} and \eqref{coupling-Markov} in combination with Lemma \ref{coupling} into the inequality (\ref{5-1}), and having $Y_t \sim \mu$, where $\mu$ be the stationary distribution of $X$, yields the following result. 

\begin{Prop}\label{stationary}
Let $x, z\in \R, s \in [0,\infty)$, $\alpha \geq 1$ and let $X^{z, s, x}$ be the solution of (\ref{Def5.1-1}) and $\mu_z^2$ the second moment of its stationary distribution. Then, for all $\alpha \geq 1$ there exists a constant $C>0$ depending on $x, z$ and $\alpha$ such that for all $t\geq s$ we have 
\begin{align*}
\abs{\E\big[(X_t^{z, s, x})^2 - \mu_z^2\big]} &\leq  C \cdot (t - s)^{-\alpha/2}.
\end{align*}
\end{Prop}
\begin{Remark}\label{stationary-r}
    With similar steps as in Proposition \ref{stationary} one can also show
    \begin{align*}
\abs{\E\big[(X_t^{z, s, x}) - \mu_z^1\big]} &\leq  C \cdot (t - s)^{-\alpha/2}
\end{align*}
for all $\alpha \geq 1$, $t\geq s$ where $\mu^1_z$ is the mean of the stationary distribution of $X^{z,s, x}$ and some constant $C>0$ depending on $x, z$ and $\alpha$.
\end{Remark}

\begin{Lemma}\label{lemma aug 16}
For all $z \in \R$ and $s, t \in [0,\infty)$ with $t \ge s$ we have  
\begin{align}
\E[ (X^{z,s,z}_t)^2] - \mu_z^2 = \E[ (X^{0,s,0}_t)^2] - \mu_0^2 + 2z (\E[X^{0,s,0}_t]+\delta).   
\end{align}
\end{Lemma}
\begin{proof}
Note that $X^{z,s,z}_t = z + X^{0,s,0}_t$, and that Formula \eqref{property_stationary_distr} implies $\mu_z^2 = \mu_0^2 - 2 z \delta + z^2$. Hence 
\begin{align*}
\E[ (X^{z,s,z}_t)^2] - \mu_z^2 & = \E[ (X^{0,s,0}_t + z)^2] - \mu_0^2 + 2 z \delta - z^2 \\
& = \E[ (X^{0,s,0}_t)^2 + 2 z X^{0,s,0}_t] - \mu_0^2 + 2 z \delta \\
& = \E[ (X^{0,s,0}_t)^2 ] - \mu_0^2 + 2z (\E[X^{0,s,0}_t]+\delta). 
\end{align*}

\end{proof}

\begin{Prop}\label{LLNcont}
Let $x, z\in \R$ and let $\mu^1_z$ be the mean of the stationary distribution of the solution $X^z = (X^{z,0,x}_t)_{t\geq 0}$ of \eqref{Def5.1-1}. 
\begin{enumerate}
\item The mean of the stationary distribution is given by \[\mu^1_z = z-\delta.\] 
    \item It holds \[\frac1T\int_0^T X_t^z\,dt \to \mu_z^1, \; \textbf{a.s.},\]
    as $T\to \infty$. 
\item For all $p > 0$ there exists a constant $C > 0$ dependent on the initial value $x = X_0^z$ such that 
\[\E\left|\frac1T\int_0^T X_t^z\,dt - \mu_z^1\right|^p \leq \frac{C}{T^{p/2}},\]
for all $T\geq 0$.
\item For $p=2$, there exists a constant $C$ such that
\[\E\left|\frac1T\int_0^T X_t^{z,0,x} - \mu_z^1\,dt\right|^2 \leq \frac{C(1+\abs{x}^{4})}{T},\]
    for all $T\geq 0$ and all initial values $x = X_0^z$.
\end{enumerate}
\end{Prop}

\begin{proof}
\begin{enumerate}
    \item Suppose the solution $X$ 
     of \eqref{Def5.1-1} starts in its stationary distribution. First note that with an easy computation it holds 
    \begin{displaymath}
        X^z_t \stackrel{d}{=} X^0_t  + z
    \end{displaymath}
    for all thresholds $z\in\R$.
    Since the optimal switching point $\delta$ minimizes the expected quadratic deviation from the origin, i.e. it minimizes 
    \begin{align}
        z \mapsto \E[(X^z_t)^2] = \E[(X^0_t + z)^2], \label{transf-1}
    \end{align}
    the point $\delta$ is a root of the first derivative of \eqref{transf-1} in $z\in \R$. Hence it holds
    \begin{align}
        0 = 2\E[(X^0_t + \delta)] = 2\left(\mu_0^1 + \delta \right).\label{transf-2}
    \end{align}
     Now by \eqref{transf-2} we have
    \begin{displaymath}
        \E\left[X_t^z\right] = \E\left[X^0_t  + z\right] = \mu_0^1  + z = z - \delta.
    \end{displaymath}
    \item[2. \& 3.] Note that for $|x|\geq |z|$ we have $\sgn (x) b_z(x)  < 0$. Thus, the condition $\mathcal R \mathcal P$ in \cite{kutoyants2004statistical} holds and Theorem 1.16. in \cite{kutoyants2004statistical} implies the almost sure convergence. The inequality follows from Proposition 1.18. 
    \item[4.]
    
    Let $f \equiv \textrm{id}$ be the identity mapping on $\R$ and $\tilde{f}(x) := f(x) - \mu^1_z$ and define   
    \begin{align}
        u(x) 
        &:=
        \begin{cases}
            \frac{\mu_z^1 x}{\theta_0}- \frac{x^2}{2\theta_0} - \frac{x}{\theta_0^2}, &x\geq z, \\
             -\frac{\mu_z^1 x}{\theta_1} - c_1 \frac{\exp((z-x)\theta_1)}{\theta_1} + \frac{x^2}{2\theta_1} - \frac{x}{\theta_1^2} + c_z, &x< z.
        \end{cases}
    \end{align}
    for the constants 
    \begin{align*}
        c_1 &:= \left(\frac{1}{\theta_0} - \frac{1}{\theta_1} - \delta \right)2\delta,\;\;\;
        c_z := z\delta \left(\left(\mu_z^1 - \delta\right) + 2\left(\frac{1}{\theta_1} - \frac{1}{\theta_0}\right)   \right)+ \frac{c_1}{\theta_1}.
    \end{align*}
    Note that the constants $c_1$ and  $c_z$ ensure that $u$ is $C^1$ on $\R$ and $C^2$ on $\R\setminus \{z\}$.
    Then an easy computation shows that $u$ solves the Poisson equations $Lu = \tilde f$, where
    \begin{displaymath}
        Lu(x) = b_z(x) \frac{\partial u}{\partial x}(x) + \frac{\partial^2 u}{\partial x^2}(x).
    \end{displaymath}
    Furthermore,
    \begin{align*}
        \abs{u(x)} &\leq C_1(1 + \abs{x}^2),\; \abs{u'(x)} \leq C_2(1 + \abs{x})
    \end{align*}
    for some constants $C_1, C_2 > 0$.
    With  It\^o's formula it then follows
    \begin{align*}
        \int_0^t \tilde{f}(X^{z,0,x}_s)ds = u(X^{z,0,x}_s) - u(X^{z,0,x}_0) -  \int_0^t u'(X^{z,0,x}_s)dW_s.
    \end{align*}
    By Lemma 1 of \cite{VERETENNIKOV1997115} it holds 
    \begin{align*}
        \E[(X^{z,0,x}_t)^4] \leq C(1 + x^4)
    \end{align*}
    for some constant $C>0$ dependent only on $z$.
    
    Hence it holds
    \begin{align*}
        \E\left[\abs{\int_0^t X^{z,0,x}_s -   \mu_z^1 ds}^2 \right] &= \E\left[\abs{\int_0^t \tilde{f}(X^{z,0,x}_s)ds}^2 \right]\\
        &\lesssim \E\left[\abs{u(X^{z,0,x}_s)}^2 \right] + \abs{u(x)}^2 + \E\left[\abs{\int_0^t u'(X^{z,0,x}_s)dW_s}^2 \right] \\
        &= \E\left[\abs{u(X^{z,0,x}_s)}^2 \right] + \abs{u(x)}^2 + \E\left[\int_0^t u'(X^{z,0,x}_s)^2ds \right] \\
        &\lesssim  \E\left[\big(1 + \abs{X^{z,0,x}_s}^2\big)^2 \right] + \big(1 + \abs{x}^2\big)^2 + \E\left[\int_0^t \big(1 + \abs{X^{z,0,x}_s}\big)^2 ds \right] \\
        &\lesssim t (1+\abs{x}^{4}).
    \end{align*}

\end{enumerate}
\end{proof}

\section{Proof of the regret rates}\label{sec: proofs of regret rates}

As in the previous sections we denote by $\mu_z^1$ and $\mu_z^2$ the first and second moments of $X^z_t$ under its stationary distribution.

We begin this section by providing a useful auxiliary identity which we will  later need for the analysis of the regrets of our algorithms for estimating the optimal switching point $\delta$.

\begin{Lemma}\label{lemma joint distri}
Let $x, z \in \R$ with $x \ge z$ and $\rho = \inf\{t \ge 0: x+W_t + \theta_0 t = z\}$. Then 
\begin{align}\label{integrated sec moment}
\E \int_0^{\rho} (x + W_t + \theta_0 t)^2 dt = 
 \frac{x^3}{3 |\theta_0|} + \frac{x^2}{2|\theta_0|^2} + \frac{x}{2|\theta_0|^3} -  \frac{z^3}{3 |\theta_0|} - \frac{z^2}{2|\theta_0|^2} - \frac{z}{2|\theta_0|^3}. 
\end{align}
\end{Lemma} 
\begin{proof}
We define $g(x) : = \frac{x^3}{3 |\theta_0|} + \frac{x^2}{2|\theta_0|^2} + \frac{x}{2|\theta_0|^3} $, $x \in \R$. It is straightforward to show that $g$ is increasing and solves the ODE 
\begin{align}\label{ode g}
\theta_0 g'(x) + \frac12 g''(x) + x^2 = 0.
\end{align} 
The It\^o formula implies for all $t \ge 0$ 
\begin{align*}
g(x + W_t + \theta_0 t) = & g(x) + \int_0^t g'(x+W_s + \theta_0 s) dW_s  \\
& + \int_0^t ( \theta_0 g'(x+W_s + \theta_0 s) + \frac12 g''(x+W_s + \theta_0 s)) ds. 
\end{align*}
By using \eqref{ode g} this can be simplified to 
\begin{align}\label{ito with g}
g(x + W_t + \theta_0 t) = & g(x) + \int_0^t g'(x+W_s + \theta_0 s) dW_s - \int_0^t (x+W_s + \theta_0 s)^2ds. 
\end{align}
For all $n \in \N$ with $n \ge x$ let $\rho_n = \inf\{t \ge 0: x+W_t + \theta_0 t = z$ or $ = n\}$. Notice that the Brownian integral in \eqref{ito with g} is a strict martingale up to $m \wedge \rho_n$, $m \in \N$, and hence its expectation at $m \wedge \rho_n$ is zero. We thus obtain 
\begin{align}\label{ito with g 2}
\E[g(x + W_{m \wedge \rho_n} + \theta_0 (m \wedge \rho_n))] = & g(x) - \E \int_0^{m \wedge \rho_n} (x+W_s + \theta_0 s)^2ds. 
\end{align}
Letting $m \to \infty$, using dominated convergence on the left hand side of \eqref{ito with g 2} and monotone convergence on the right hand side, we arrive at 
\begin{align}\label{ito with g 3}
\E[g(x + W_{\rho_n} + \theta_0 \rho_n)] = & g(x) - \E \int_0^{\rho_n} (x+W_s + \theta_0 s)^2ds. 
\end{align} 
It is well-known that for $n > x$ 
\begin{align*}
\P(x + W_{\rho_n} + \theta_0 \rho_n = n) = e^{\theta_0 (n-x)} \frac{\sinh((x-z)|\theta_0|)}{\sinh((n-z)|\theta_0|)}
\end{align*}
(see, e.g., Formula 3.0.4 in \cite{borodin2002handbook}). This implies that the sequence $g(x + W_{\rho_n} + \theta_0 \rho_n)$, $n \in \N$, is uniformly integrable. Therefore, the left hand side of \eqref{ito with g 3} converges, as $n \to \infty$, to $ E[g(x + W_{\rho} + \theta_0 \rho)] = g(z)$.  
The right hand side of \eqref{ito with g 3} converges by monotone convergence, and hence we get 
\begin{align*}
g(z) = g(x) -  \E \int_0^{\rho} (x+W_s + \theta_0 s)^2ds, 
\end{align*} 
which is equivalent to \eqref{integrated sec moment}.  
\end{proof}

\subsection{Proof of Theorem \ref{thm 1}}

\begin{proof}[Proof of Theorem \ref{thm 1}]

Let $X = (X_t)_{t\geq 0}$ be the state dynamics if the process is controlled with the explore-first algorithm. W.l.o.g. we assume that the process starts in $x = 0$. Thus $X_t = X_t^0 =  X_t^{0, 0, 0}$ for $t\in [0, \tau)$ and $X_t = X_t^{\hat{\delta}, \tau, X_\tau}$ for $t \in [\tau, T]$. Here the threshold $\hat{\delta}$ is the estimate of the optimal switching point $\delta$, defined in \eqref{defi delta}, using the information of the state $X_t$ up to $\tau$.  In the following let $C>0$ denote a generic constant that only depends on $\theta_0$ and $\theta_1$. 

For $T \ge \tau$ the regret can be decomposed into 
\begin{align*}
R^{\text{ex first}}(T) = I + J + K,  
\end{align*}
where 
\begin{align*}
I := \E \int_0^\tau (X^0_t)^2 - \mu_\delta^2 dt, \qquad J := \E \int_\tau^T (X_t^{\hat{\delta}, \tau, X_\tau})^2 - \mu_\delta^2 dt,\qquad K := \E\int_0^T (\mu_\delta^2 - (X^{\delta}_t)^2) dt.
 \end{align*}
First note that $K \le \int_0^T |\E [(X^{\delta}_t)^2] - \mu_\delta^2| dt$. Then Proposition \ref{stationary} yields that for all $t > 0$ we have $|\E [(X^{\delta}_t)^2] - \mu_\delta^2| \le C/\sqrt{t}$. Consequently, 
\begin{align}\label{esti K}
K \le 2 C \sqrt{T}.
\end{align}
Next observe that 
\begin{align}\label{first esti I}
I & = \int_0^\tau \E [(X^{0}_t)^2 - \mu_\delta^2] dt \nonumber\\
& \le \int_0^\tau |\E [(X^{0}_t)^2] - \mu_0^2| dt + \int_0^\tau |\mu_0^2 - \mu_\delta^2| dt. 
\end{align}
Formula \eqref{property_stationary_distr} implies $\mu_0^2 - \mu_\delta^2 = \delta^2$, and hence the second term in \eqref{first esti I} is equal to $\delta^2 \tau$. Again by Proposition \ref{stationary}   we have $|\E [(X^{0}_t)^2] - \mu_0^2| \le C/\sqrt{t}$ for all $t > 0$. Consequently, the first term in \eqref{first esti I} is smaller than or equal to $2 C \sqrt{\tau}$. Thus, we get
\begin{align}\label{esti I}
I \le \delta^2 \tau + 2C \sqrt{\tau}. 
\end{align}
Next we estimate $J$.  
As previously, our aim is apply Proposition \ref{stationary} in order to estimate $\E[(X_t^{\hat{\delta}, \tau, X_\tau})^2] - \mu_\delta^2$. 
We first let the process return to threshold level $\hat \delta$. More precisely,  let 
\begin{align*}
\sigma:= \inf\{t \ge \tau: X_t^{\hat{\delta}, \tau, X_\tau} = \hat \delta\} \wedge T.
\end{align*}
Notice, in the case the event $\{\sigma < T\}$ occurs, $\sigma$ is the first  time that $X_t^{\hat{\delta}, \tau, X_\tau}$ attains the level $\hat \delta$ after $\tau$ and before the time-horizon $T$.  Moreover, for all $t \ge \sigma$ we have $X_t = X^{\hat \delta, \sigma, \hat \delta}_t$. 

We now decompose $J$ into $J = J_1 + J_2 + J_3$, where 
\begin{align*}
J_1 = \E\int_\tau^\sigma (X_t^{\hat{\delta}, \tau, X_\tau})^2 dt, \qquad  J_2 = \E\int_\sigma^T (X^{\hat \delta, \sigma, \hat \delta}_t)^2 dt  - \E\int_\tau^T \mu_{\hat \delta}^2 dt, \qquad J_3 = (T-\tau) (\E[\mu_{\hat \delta}^2]-\mu_\delta^2).
\end{align*}
We next show that $J_1$ is bounded by some constant depending only on $\theta_0$ and $\theta_1$. To this end observe that if $X^0_\tau > \hat \delta$, then on $[\tau, \sigma]$ the process $X$ is  Brownian motion with negative drift $\theta_0$. The stopping time $\sigma$ is thus the first passage time of a Brownian motion with drift. By the Markov property there exists a function $f:\R \times \R \to \R$ such that  
\begin{align*}
\E\left[\int_\tau^\sigma (X_t^{\hat{\delta}, \tau, X_\tau})^2 dt|\mathcal{F}_\tau \right] = f(X^0_\tau, \hat \delta),  
\end{align*}
where $(\cF_t)$ denotes the filtration generated by the underlying Brownian motion driving $X$.  
For $x \ge z$ the function $f$ satisfies 
\begin{align*}
f(x,z) = \E \int_0^{\rho_z} (x + W_t + \theta_0 t)^2 dt,
\end{align*}
where $\rho_z = \inf\{t \ge 0: x + W_t + \theta_0 t = z\}$. For $x < z$ we have a similar representation. Lemma \ref{lemma joint distri} implies $f(x,z) \le C (|x|^3 + |z|^3)$ . Therefore, with Proposition \ref{LLNcont} and Lemma \ref{moments} we obtain 
\begin{align}\label{esti J1}
J_1 = \E[f(X^0_\tau, \hat \delta)] \leq C. 
\end{align}

In order to handle the term $J_2$, notice first that the strong Markov property of $X$ implies 
\begin{align*}
\E\left[ \int_\sigma^T (X^{\hat \delta, \sigma, \hat \delta}_t)^2 dt\right]
&\le \E\left[1_{\{\sigma < T\}} \int_\sigma^{T+\sigma-\tau} (X^{\hat \delta, \sigma, \hat \delta}_t)^2 dt\right]\\ 
&=  \E\left[1_{\{\sigma < T\}} \E\left[\int_\sigma^{T+(\sigma-\tau)} (X^{\hat \delta, \sigma, \hat \delta}_t)^2 dt\big| \mathcal{F}_{\sigma} \right]\right]  \\
&= \E\left[ 1_{\{\sigma < T\}}\left(\int_0^{T-\tau} E[(X^{z, 0, z}_t)^2] dt \right)_{z = \hat \delta} \right]\\
&\leq \E\left[\left(\int_0^{T-\tau} E[(X^{z, 0, z}_t)^2] dt \right)_{z = \hat \delta} \right].
 \end{align*}
Thus 
\begin{align}\label{esti J2}
\E\left[ \int_\sigma^T (X^{\hat \delta, \sigma, \hat \delta}_t)^2 dt\right] - \E \int_\tau^T \mu_{\hat \delta}^2 \, dt
& \le \E\left[ \left(\int_0^{T-\tau} (\E[(X^{z, 0, z}_t)^2] - \mu_z^2) dt\right)_{z = \hat \delta} \right]
 \end{align}
By Lemma \ref{lemma aug 16} the integrand on the right hand side of \eqref{esti J2} can be simplified to  $\E[(X^{0}_t)^2] - \mu_0^2 + 2z (\E[X^{0}_s]+\delta)$. We thus obtain 
\begin{align}\label{esti J2 2}
\E\left[ \int_\sigma^T (X^{\hat \delta, \sigma, \hat \delta}_t)^2 dt\right] - \E \int_\tau^T \mu_{\hat \delta}^2 \, dt
& \le \int_0^{T-\tau} (\E[(X^{0}_t)^2] - \mu_0^2) dt + 2 \E[\hat \delta] \E\int_0^{T-\tau} (X^{0}_t + \delta) dt. 
 \end{align}
The first term on the right hand side of \eqref{esti J2 2} can be estimated with Proposition \ref{stationary} from above by $C / \sqrt{T- \tau}$. For the second term we again use Proposition \ref{LLNcont} guaranteeing  $\E[\hat\delta ] \le C$ and  
\begin{align*}
    \left|\int_0^{T-\tau} (\E[X^{0}_t] + \delta) dt\right| \le C \sqrt{T - \tau}. 
\end{align*}
We thus obtain 
\begin{align*}
J_2 = \E\left[ \int_\sigma^T (X^{\hat \delta, \sigma, \hat \delta}_t)^2 dt\right] - \E \int_\tau^T \mu_{\hat \delta}^2 dt \le 2 C \sqrt{T-\tau} + 2 C^2 \sqrt{T - \tau}. 
\end{align*}
Finally, note that by Formula \eqref{property_stationary_distr} we have $J_3 = (T-\tau) \E[(\hat \delta - \delta)^2]$. Again by Proposition \ref{LLNcont} we have that $\E[(\hat \delta - \delta)^2]  \le C /\tau$, and hence we get 
\begin{align}\label{esti J3}
J_3 \le C \frac{T- \tau}{\tau}. 
\end{align}
Putting together the estimates \eqref{esti K}, \eqref{esti I}, \eqref{esti J1}, \eqref{esti J2} and \eqref{esti J3}, we arrive at 
\begin{align}\label{esti ef regret final}
R^{\text{ex first}}(T) \le 2 C \sqrt{T}  + \delta^2 \tau + 2C \sqrt{\tau} + C + 2 C \sqrt{T-\tau}+ 2 C^2 \sqrt{T - \tau} + C \frac{T- \tau}{\tau}.
\end{align} 
By choosing $\tau = \sqrt{T}$ on the right hand side of \eqref{esti ef regret final} we get that there exists a constant $C$, depending only on $\theta_0$ and $\theta_1$, such that $R^{\text{ex first}}(T) \le C(1 + \sqrt{T})$. 
\end{proof}

\subsection{Proof of Theorem \ref{regret}}


In the following we use the notation $ X^z_t = X^{z,s,x}_t$ for the state process $(X^{z,s,x}_t)_{t\geq s}$ if the values $s$ and $x$ are clear from the context.

Define the sequence of estimators $(\hat \delta_n)_{n\geq 0}$ with which we control the  process $(X_t^{\hat \delta})_{t\geq 0}$ with $X_{0}^{\hat \delta}= x_0 \in \R$ as follows
\begin{equation}
    \begin{aligned}\label{APAC-Def}
X_t^{\hat \delta} &=  X_t^{\hat \delta_n, \tau_n, x_n}, \quad t\in I_n, n\in \N_0,
\end{aligned}
\end{equation}
with 
\begin{itemize}
    \item an increasing time sequence with $\tau_0 = 0$, $\tau_{n+1} := 2^{n+1}$  and time-intervals 
$    
        I_n := [\tau_n, \tau_{n+1}) 
$    
    \item initial states 
$
        x_n := X_{\tau_n}^{\hat \delta_{n-1}, \tau_{n-1}, x_{n-1}},
$
    \item state process $X_t^{\hat \delta_n, \tau_n, x_n}$ defined as the solution of \eqref{Def5.1-1} on the respective time-interval $I_n$,
    \item and lastly the sequence of estimators of the optimal control threshold $\delta$
    \begin{displaymath}
        \hat \delta_{n+1} = \hat \delta_n - \frac{1}{\Delta \tau_n} \int_{I_n} X_t^{\hat \delta_n}\,dt,
    \end{displaymath}
    where $\Delta \tau_n := \tau_{n+1} - \tau_n$ and  $n\in\N_0$.
\end{itemize}
Note that we have
\[X^z \stackrel{d}{=} X^0  + z, \quad z\in \R,\]
and therefore we can rewrite the recursion of $\hat{\delta}$ as
\[\hat \delta_{n+1} \stackrel{d}{=} - \frac{1}{\Delta \tau_n} \int_{I_n} X^0_t dt,\]
for $n\in \N$.  This will play a key part in the proof of Theorem \ref{regret}.

As for the regret of the explore-first algorithm, we would like to find a bound of the expected regret
\begin{displaymath}
R^{\text{APAC}}(T) = \E\left[\int_0^T \left(X_t^{\hat \delta}\right)^2 - \left(X_t^{\delta}\right)^2dt\right]
\end{displaymath}
 up to time $T$. As in Section \ref{Ch-ICP} the process $X_t^{\delta}$ denotes the optimally controlled process starting at time $0$ having the same initial value as $X_t^{\hat \delta}$ .  The following proof shows that the regret of the adaptive algorithm is indeed bounded and its order is lower than the previous algorithm.\\

\begin{proof}[Proof of Theorem \ref{regret}]
    First let us fix a time horizon $T>0$. Let  the sequence $\tau_{n}$ be defined as in \eqref{APAC-Def} for $n\in \{0,1,\ldots, N\}$ where $N := \ceil{\log(T+1)}$. Then
    \begin{displaymath} 
        \sum_{n = 0}^{N-1} \Delta \tau_n = \sum_{n = 0}^{N-1} 2^n  = 2^N -1 \geq T,
    \end{displaymath}
    hence $[0,T]\subset \bigcup_{k=0}^{N-1} I_n$ for the time-intervals $I_n = [\tau_n,\tau_{n+1})$ defined as in \eqref{APAC-Def}.  
    We then split the regret into three parts:  
    \begin{align}
        R^{\text{APAC}}(T)    &= \E\left[\int_0^T \left(X_t^{\hat \delta}\right)^2 - \left(X_t^{\delta, 0}\right)^2dt\right]\nonumber \\ 
                &\leq \E\left[ \sum_{k=0}^{N-1} \int_{\tau_k}^{\tau_{k+1}}\big(X_t^{\hat{\delta}_k}\big)^2 - \mu_{\hat{\delta}_k}^2 dt \right] \label{regret1}\\
                &\quad + \E\left[ \sum_{k=0}^{N-1} \int_{\tau_k}^{\tau_{k+1}} \mu_{\hat{\delta}_k}^2  - \mu_{\delta}^2dt \right] \label{regret2}\\
                &\quad + \E\left[ \int_{0}^{T} \mu_{\delta}^2 - \left(X_t^{\delta, 0}\right)^2 dt \right]. \label{regret3}
    \end{align}

    The three parts can be described intuitively as follows: 
\begin{itemize}
\item \textbf{Part (\ref{regret1})}:  The cost/regret of the process not being stationary when using the estimate $\hat{\delta}$ as the control. 
\item \textbf{Part (\ref{regret2})}: The cost/regret of using the estimate $\hat{\delta}_k$ instead of the optimal $\delta$, when started from the stationary distribution.
\item \textbf{Part (\ref{regret3})}: The cost/regret of the process not being stationary when using the optimal control $\delta$. 
\end{itemize}


    In the following let $C>0$ denote a generic constant that only depends on $\theta_0$, $\theta_1$ and $K$.
    We show that the term (\ref{regret2}) is bounded (up to a constant) by $\log(T)$. 
    First, let us introduce two Brownian motions, reflected on the bounds $K$ and $-K$ respectively, defined by the SDEs
    \begin{align*}
        Y_t^+ &= \theta_0 dt + dW_t + dL_t^{\theta_0}, \\
        Y_t^- &= \theta_1 dt + dW_t - dL_t^{\theta_1},
    \end{align*}
    with initial values $Y_0^+ = K$ and $Y_0^- = -K$ and local times $L_t^{\theta_i}$ fulfilling
    \begin{displaymath}
        \int_0^t 1_{\{Y_t^+ > K\}} dL_t^{\theta_0} = 0, \qquad \int_0^t 1_{\{Y_t^- < -K\}} dL_t^{\theta_1} = 0.
    \end{displaymath}
    Here the Wiener process $(W_t)$ is the same process as in the definition of $X^{\hat{\delta}}_t$.
    Note that reflected Brownian motions with drift towards the boundary have finite $p$-th moments, see \cite{borodin2002handbook} Appendix 16. 
    Notice that by construction it holds $\hat{\delta}_k \in [-K, K]$. Then,  by using a comparison argument, one can show that
     $Y_t^- \leq X^{\hat{\delta}_k}_t \leq Y_t^+$ for all $t\geq 0$ and hence
    \begin{align}\label{refl-upbound2}
        \E\left[|X^{\hat{\delta}_k}_t|^p\right] \leq \E\left[|Y_t^+|^p + |Y_t^-|^p\right] \leq C.
    \end{align}

    For the first part (\ref{regret1}) we will proceed as in the proof of Theorem \ref{thm 1} with shorter arguments.  We decompose 
    \begin{align*}
        J = \E\left[ \int_{\tau_k}^{\tau_{k+1}}\big(X_t^{\hat{\delta}_k}\big)^2 - \mu_{\hat{\delta}_k}^2 dt\right]
    \end{align*}
    into two parts 
    \begin{align*}
        J_1 = \E \left[\int_{\tau_k}^{\sigma_k}\big(X_t^{\hat{\delta}_k}\big)^2 dt\right] \quad \textrm{ and }\quad J_2 = \E \left[\int_{\sigma_k}^{\tau_{k+1}}\big(X_t^{\hat{\delta}_k}\big)^2 dt\right] - \E \left[\int_{\tau_k}^{\tau_{k+1}}\mu_{\hat{\delta}_k}^2 dt\right],
    \end{align*}
    where $\sigma_k := \inf\{t \geq \tau_k : X_t^{\hat{\delta}_k} = \hat{\delta}_k \} \wedge \tau_{k+1}$ is the first 
    return time before $\tau_{k+1}$ to the threshold $\hat{\delta}_k$ after updating the estimator to $\hat{\delta}_k$. W.l.o.g. suppose  $X_t^{\hat{\delta}_k} \geq \hat{\delta}_k$. Then on $[\tau_k, \sigma_k]$ the state process $X^{\hat{\delta}_k}$ is a Brownian motion with negative drift $\theta_0$. 
    
    By the Markov property we then have 
    \begin{align*}
        \E\left[\int_{\tau_k}^{\sigma_k} \big(X_t^{\hat{\delta}_k} \big)^2 dt\right] &= \E\left[ \E\left[\int_{\tau_k}^{\sigma_k} \big(X_t^{\hat{\delta}_k} \big)^2 dt \big| \mathcal{F}_{\tau_k}\right]\right]  \\ 
        &= \E\left[\E \int_0^{\tau_{\hat{\delta}_k}}\left(X_{\tau_k}^{\hat{\delta}_k} + W_t + \theta_0t\right)^2 dt\right] \\
        &\leq  C \E\left[\left(X_{\tau_k}^{\hat{\delta}_k}\right)^3 +(\hat{\delta}_k)^3 \right].
    \end{align*}
    
     Proposition \ref{LLNcont} implies $\E[(\hat{\delta}_k)^3] < C$ and hence together with \eqref{refl-upbound2} we have 
    \begin{align}
        J_1 \leq \E\left[\left(X_{\tau_k}^{\hat{\delta}_k}\right)^3 +\big(\hat{\delta}_k\big)^3 \right] \leq C. \label{APAC-reg-J1}
    \end{align}
    
    For the second part $J_2$ we have by the Markov property 
    \begin{align*}
        \E \left[\int_{\sigma_k}^{\tau_{k+1}}\big(X_t^{\hat{\delta}_k}\big)^2 dt\right] & \leq \E \left[1_{\{\sigma_k < \tau_{k+1}\}} \int_{\sigma_k}^{\tau_{k+1} + \sigma_k - \tau_k}\big(X_t^{\hat{\delta}_k}\big)^2 dt\right] \\
         & = \E \left[1_{\{\sigma_k < \tau_{k+1}\}} \E\left[\int_{\sigma_k}^{\tau_{k+1} + \sigma_k - \tau_k}\big(X_t^{\hat{\delta}_k}\big)^2 dt\big| \cF_{\sigma_k} \right]\right] \\
          & = \E \left[1_{\{\sigma_k < \tau_{k+1}\}} \left(\int_{0}^{\tau_{k+1}-\tau_k}\E\big(X_t^{z,0,z}\big)^2 dt\right)_{z=\hat{\delta}_k}\right] \\
        & \leq  \E \left[\left(\int_{0}^{\tau_{k+1}-\tau_k}\E\big(X_t^{z, 0, z}\big)^2 dt\right)_{z=\hat{\delta}_k}\right]
    \end{align*}
    and hence as in the proof of Theorem \ref{thm 1} 
    \begin{align}
        J_2 = \E \left[\int_{\sigma_k}^{\tau_{k+1}}\big(X_t^{\hat{\delta}_k}\big)^2 dt\right] - \E \left[\int_{\tau_k}^{\tau_{k+1}}\mu_{\hat{\delta}_k}^2 dt\right] \leq \E\left[\left (\int_0^{\tau_{k+1}- \tau_k}\E\big(X_t^{z, 0, z}\big)^2 - \mu_z^2 dt\right)_{z=\hat{\delta}_k}\right].\label{APACreg1}
    \end{align}
    Now by applying Lemma \ref{lemma aug 16} the right hand side of (\ref{APACreg1}) simplifies to 
    \begin{align}
        J_2 \leq \int_0^{\tau_{k+1}- \tau_k} \big( \E\left[(X_t^{0})^2\right] - \mu_0^2 \big)dt + 2\E[\hat{\delta}_k]\E\left[ \int_0^{\tau_{k+1}- \tau_k} \big(X_t^0 + \delta\big) dt\right].\label{APACreg2}
    \end{align}
    For the first term of the right hand side of (\ref{APACreg2}) we can now apply Proposition \ref{stationary} with $\alpha > 2$ 
    \begin{align}
        \int_0^{\tau_{k+1}- \tau_k} \big( \E\left[(X_t^{0})^2\right] - \mu_0^2 \big)dt &\leq C + \int_1^{\tau_{k+1}- \tau_k}\E\left[(X_t^{0})^2 - \mu_0^2\right]dt \nonumber \\
        &\leq C + \int_1^{\tau_{k+1}- \tau_k} t^{-\frac{\alpha}{2}}dt\nonumber \\
	&= C + \frac{1}{1-\frac{\alpha}{2}} \left(\big(\Delta\tau_k\big)^{1-\frac{\alpha}{2}} - 1 \right).\label{APACreg3}
    \end{align}
    Now by taking $\alpha = 3$ (or larger $\alpha$) the term (\ref{APACreg3}) can be bounded by some constant for large $k$. We now turn to the second summand of the right hand side of (\ref{APACreg2}).
    By analogous steps as in (\ref{APACreg3}) together with Remark \ref{stationary-r} 
    we have
    \begin{align*}
    	\E\left[\int_0^{\tau_{k+1}- \tau_k} \big(X_t^0 + \delta \big) dt\right] \leq C. 
    \end{align*}
    By Proposition \ref{LLNcont} we also have $\E[\hat{\delta}_k] < C$. Hence we have $J_2 \leq C$.
    Combining this with \eqref{APAC-reg-J1}, we obtain that 
    part (\ref{regret1}) is bounded by
    \begin{align}
        \E\left[ \sum_{k=0}^{N-1} \int_{\tau_k}^{\tau_{k+1}}\big(X_t^{\hat{\delta}_k}\big)^2 - \mu_{\hat{\delta}_k}^2 dt \right]  \leq \sum_{k=0}^{N-1} \big(J_1 + J_2\big) \leq \sum_{k=0}^{N-1} C \leq C\big(\log(T) + 1\big). \label{APACreg5}
    \end{align}

    We show that the second part \eqref{regret2} is also bounded by $C\cdot\log(T)$.
    Note that by Proposition \ref{LLNcont}, part 1., it holds  $\delta = -\mu_0^1$.
    Using this together with Lemma \ref{property_stationary_distr} and Proposition \ref{LLNcont}, part 4., we have
    \begin{align*}
      \Delta \tau_n \E\left[\abs{ \frac{1}{\Delta \tau_{n-1}} \int_{\tau_{n-1}}^{\tau_{n}} X_t^0 dt - \mu_0^1}^2\right] 
     &= \Delta \tau_n  \E\left[\E\left[\left.\abs{ \frac{1}{\Delta \tau_{n-1}} \int_{\tau_{n-1}}^{\tau_{n}} X_t^0 - \mu_0^1 dt }^2\right| X^{0,\tau_{n-1}, x_{n-1}}_{\tau_{n-1}} \right] \right] \\
     &\leq C
     \frac{\Delta \tau_n}{\Delta \tau_{n-1}} \E\left[1
      + \big(X^{0,\tau_{n-1}, x_{n-1}}_{\tau_{n-1}}\big)^{4} \right] \\
      &\leq C \frac{\Delta \tau_n}{\Delta \tau_{n-1}} \E\left[1 +( Y_t^+)^4 + (Y_t^-)^4\right]\\
     &\leq  \frac{\Delta \tau_n}{\Delta \tau_{n-1}}C.
\end{align*}

Since there are $\ceil{ \log_2(T)}$ summands in (\ref{regret2}) we have 
\begin{align}
    \E\left[ \sum_{k=0}^{N-1} \int_{\tau_k}^{\tau_{k+1}} \mu_{\hat{\delta}_k}^2  - \mu_{\delta}^2dt \right] \leq \sum_{k=0}^{N-1} 2\cdot C \leq 2  C(\log_2(T)+1). \label{rpart2}
\end{align}

  Finally, we turn to the estimation of \eqref{regret3}. Note that
    Proposition \ref{stationary} with $\alpha = 2$ yields 
    \begin{align}
        \E\left[ \int_{0}^{T} \left(\mu_{\delta}^2 - X_t^{\delta, 0}\right)^2 dt \right] &= \E \left[\int_0^1 \mu_{\delta}^2 - \left(X_t^{\delta, 0}\right)^2 dt\right] +  \int_{1}^{T} \E\left[ \left(X_t^{\delta, 0}\right)^2  - \mu_{\delta}^2 \right] dt \nonumber\\
        &\leq C \left(1+  \int_1^T \frac{1}{t}dt \right) \nonumber\\
        &= C\left( 1 + \log_2(T)\right).\label{rpart3}
    \end{align} 
    Combining parts (\ref{APACreg5}), (\ref{rpart2}) and (\ref{rpart3}) we conclude \eqref{regret-eq}.

\end{proof}


\printbibliography

\end{document}